\documentclass[pdflatex,sn-basic,Numbered]{sn-jnl}
\usepackage[numbers]{natbib}

\usepackage[numbers,sort&compress]{natbib}
\usepackage{graphicx}%
\usepackage{multirow}%
\usepackage{amsmath,amssymb,amsfonts}%
\usepackage{amsthm}%
\usepackage{mathrsfs}%
\usepackage[title]{appendix}%
\usepackage{xcolor}%
\usepackage{textcomp}%
\usepackage{manyfoot}%
\usepackage{booktabs}%
\usepackage{algorithm}%
\usepackage{algorithmicx}%
\usepackage{algpseudocode}%
\usepackage{listings}%

\usepackage{anyfontsize}

\renewcommand{\figurename}{Figure}

\newtheorem{theorem}{Theorem}[section]
\newtheorem{corollary}[theorem]{Corollary}
\newtheorem{lemma}[theorem]{Lemma}

\newtheorem{definition}{Definition}[section]%

\begin{document}

\title[Article Title]{Frame Quantization of Neural Networks}




\author[1]{\fnm{Wojciech}\sur{Czaja}}
\author*[1]{ \fnm{Sanghoon}\sur{Na}}\email{shna2020@umd.edu}

\affil[1]{\orgdiv{Department of Mathematics}, \orgname{University of Maryland}, \\
\city{College Park}, \postcode{MD 20742}, 
\country{USA}}




\abstract{We present a post-training quantization algorithm with error estimates relying on ideas originating from frame theory. Specifically, we use first-order Sigma-Delta ($\Sigma\Delta$) quantization for finite unit-norm tight frames to quantize weight matrices and biases in a neural network. In our scenario, we derive an error bound between the original neural network and the quantized neural network in terms of step size and the number of frame elements. We also demonstrate how to leverage the redundancy of frames to achieve a quantized neural network with higher accuracy.}

\keywords{Neural Network Quantization, Post-Training Quantization, Sigma-Delta Quantization, Finite Frames}



\maketitle

\section{Introduction}\label{sec1}
\textit{Quantization} is the process of compressing input from a continuous or large set of values into a small-sized discrete set. It gained popularity in \textit{signal processing}, where one of its primary goals is obtaining a condensed representation of the analogue signal suitable for digital storage and recovery. Examples of quantization algorithms include truncated binary expansion, pulse-code modulation (PCM) and sigma-delta ($\Sigma\Delta$) quantization. Among them, $\Sigma\Delta$ algorithms stand out due to their theoretically guaranteed robustness. Mathematical theories were developed in several seminal works  \cite{daubechies2003approximating,gunturk2004approximating,benedetto2006sigma,benedetto2008complex,benedetto2006second}, and have been carefully studied since, e.g., \cite{gunturk2013sobolev,krahmer2014sigma,saab2018quantization,gunturk2010sigma}.

In recent years, the concept of quantization also captured the attention of the machine learning community. The quantization of \textit{deep neural networks} (DNNs) is considered one of the most effective network compression techniques \cite{gholami2022survey}. 
Computers express parameters of a neural network as 32-bit or 64-bit floating point numbers. In \textit{neural network quantization}, one tries to replace these parameters using compact formats such as 8 bits (or lower), while preserving the architecture and performance of the network. An effective neural network quantization method enables users to run DNNs on portable devices, such as smartphones, without relying on external servers. This benefits storage requirements and helps to avoid privacy-related issues. Due to these reasons, there have been numerous efforts to develop neural network quantization schemes that preserve the model accuracy. 

Neural network quantization can be categorized into 2 classes. The first class is \textit{Quantization-Aware-Training} (QAT). QAT methods \cite{courbariaux2015binaryconnect,hubara2016binarized,rastegari2016xnor,yin2018binaryrelax,kummer2023adaptive,long2023recurrence} retrain the given neural network, by restricting the domain of parameters to a finite set of alphabets. The second class is \textit{Post-Training Quantization} (PTQ). PTQ methods \cite{nagel2020up,wei2022qdrop,nahshan2021loss,guo2022squant,lybrand2021greedy,zhang2023post,banner2019post,yao2022zeroquant,zhang2023spfq,maly2023simple} take a pre-trained neural network and convert it directly into a fixed-point neural network. They demand less computation because they do not require end-to-end training. In addition, unlike QAT methods, PTQ methods do not require the training dataset - they usually require only a small calibration set. These reasons make PTQ methods attractive. For a detailed explanation of QAT and PTQ methods, see \cite{gholami2022survey} and references therein.

Despite the popularity of PTQ methods, unfortunately, most of them lack theoretical error analysis. While some algorithms \cite{lybrand2021greedy,zhang2023post,zhang2023spfq,maly2023simple} are equipped with error estimates, they are typically probabilistic estimates with additional requirements on input distributions. Moreover, their error estimates are proved only for feed-forward networks. However, there are popular neural networks with different architectures, such as ResNet \cite{he2016deep}, which is constructed by stacking together a series of residual blocks. Therefore, the design of a PTQ method with error estimates for different types of neural networks is still an open question.

To address this problem, we return to the methods of quantization from signal processing and investigate how to utilize them in the neural network setting while retaining their theoretical guarantees. We focus on $\Sigma\Delta$ quantization algorithms. While \cite{daubechies2003approximating,gunturk2004approximating} focus on $\Sigma\Delta$ quantization for bandlimited functions, \cite{benedetto2006sigma,benedetto2006second,benedetto2008complex} provide mathematical analysis of $\Sigma\Delta$ quantization for finite frames in $\mathbb{R}^d$ and $\mathbb{C}^d,$ which is suitable for quantization of weight matrices in neural networks. Leveraging the tools from \cite{benedetto2006sigma,benedetto2006second,benedetto2008complex}, we are able to provide a new PTQ algorithm with error estimates. Our contributions are threefold: 

\begin{enumerate}
    \item In Section \ref{ouralg} we propose an algorithm for quantizing layers of a pre-trained neural network using the first-order $\Sigma\Delta$ quantization algorithm for finite unit-norm tight frames in $\mathbb{R}^d.$ Our algorithm does not require any training data, hence a data-free quantization. To the best of our knowledge, our work is the first to utilize the tools from frame theory in this context (recent work  \cite{zhang2024sigma} studies quantizing random Fourier features (RFFs) using $\Sigma\Delta$ algorithm, which is different from our study of DNNs).
    \item We provide error estimates for both $n$-layer feed-forward neural networks and neural networks formed by a series of residual blocks in Section \ref{errorestimates}. These results demonstrate how to control the error using the step size and the number of frame elements. 
    \item We provide numerical results for quantizing neural networks trained on the MNIST data set using our proposed algorithm in Section \ref{numresults}. We quantize all layers and demonstrate that our algorithm can even perform 1-bit quantization with a small loss of accuracy.
\end{enumerate}

\section{Preliminaries}\label{Prelim}
In this paper, all vectors are column vectors in $\mathbb{R}^d$, $d \in \mathbb{N}.$ For a vector $x$, $\|x\|$ denotes the Euclidean norm of $x$. For a matrix $A$, $\|A\|$ denotes the matrix 2-norm of $A$, which is the largest singular value of $A$.

We introduce the mathematical formulation of neural networks used in this paper. First, a fully connected \textit{feed-forward neural network} (FNN) with $n$ layers is a function $f:\mathbb{R}^{m_0}\rightarrow\mathbb{R}^{m_n}$ which acts on data $x \in \mathbb{R}^{m_0}$ via
\begin{align}\label{fnn}
        f(x)=h^{[n]}\circ\sigma\circ h^{[n-1]}\circ\cdots\circ\sigma\circ h^{[1]}(x),
\end{align}
where $h^{[l]}(x)=W_l x+b_l$ with \textit{weight matrix} $W_l \in \mathbb{R}^{m_l \times m_{l-1}}$ and \textit{bias} $b_l \in \mathbb{R}^{m_l}$, $l=1,2,\cdots,n,$ and nonlinear activation function $\sigma$ which acts on each component of a vector. Here, we omit the activation for the last layer because the choice of the final activation can be different from $\sigma$. For example in classification problems, softmax activation function is used in the last step, while $\sigma$ is often chosen as the \textit{Rectified Linear Unit} (ReLU) function, defined as
$x \mapsto \max\{0,x\}$ for $x\in\mathbb{R}$. Note that ReLU is a 1-Lipchitz continuous function attaining $0$ at $x=0$. Throughout this paper we will assume activation function $\sigma$ is an $L$-Lipschitz continuous function with $\sigma(0)=0,$ which makes ReLU a special case. In addition, we will also assume that $m_i \ge 3$ for all $i$'s. This is a natural assumption since typical neural network architectures set the number of neurons at each layer to be greater than $2$.

Next we introduce the mathematical formulation for residual blocks and residual neural networks used in this paper. We adopt the mathematical formulation from Chapter 3.8.2 in \cite{grohs2022mathematical}. A \textit{residual neural network} with $n$ residual blocks is a function $g:\mathbb{R}^k \rightarrow \mathbb{R}^k$ defined as
\begin{align}\label{residual}
    g(x)=z^{[n]}\circ \sigma \circ z^{[n-1]} \circ \cdots \circ \sigma \circ z^{[1]}(x),
\end{align}
where each \textit{residual block} $z^{[i]}$ has a structure $z^{[i]}(x)=W_{i,2}\sigma(W_{i,1}x+b_i)+x,$ with weight matrices $W_{i,1},W_{i,2} \in \mathbb{R}^{k\times k}$ and bias $b_i \in \mathbb{R}^k.$ For residual networks, we only consider ReLU as the activation function.

For a neural network $F(x)$, the \textit{quantized neural network} $F_Q(x)$ is defined as a network formed by replacing weight matrices and biases of $F(x)$ with quantized weight matrices and quantized biases from any quantization algorithm. For example, for a fully connected FNN $f(x)$ in (\ref{fnn}), assume that $Q_i$'s and $c_i$'s are quantized weight matrices and quantized biases of $W_i$'s and $b_i$'s, respectively. Then,
\begin{align}\label{quant_fnn}
    f_Q(x)=h^{[n]}_Q\circ \sigma \circ h^{[n-1]}_Q \circ \sigma \cdots \circ \sigma \circ h^{[1]}_Q(x),
\end{align}
where $h^{[i]}_Q(x)=Q_ix+c_i$ for $i=1,\cdots,n.$ Similarly for a residual network $g(x)$ in (\ref{residual}),
\begin{align}\label{quant_res}
    g_Q(x)=z^{[n]}_Q\circ \sigma \circ z^{[n-1]}_Q \circ \cdots \circ \sigma \circ z^{[1]}_Q(x),
\end{align}
where $z^{[i]}_Q(x)=Q_{i,2}\sigma(Q_{i,1}x+c_i)+x$ for $i=1,\cdots,n.$ Here, $Q_{i,2}$'s, $Q_{i,1}$'s and $c_i$'s are quantized weight matrices and quantized biases of $W_{i,2}$'s, $W_{i,1}$'s and $b_i$'s respectively. 

Note that if we write $\widetilde{W}_l=(W_l,b_l)$ and $\Tilde{x}=(x^T,1)^T,$ then $h^{[l]}(x)=W_l x+b_l=\widetilde{W}_l\Tilde{x}$ holds. Therefore, without loss of generality, we will omit the biases in our analysis. This means we assume $b_i=0$ for all $i$'s, so $h^{[i]}(x)$'s in (\ref{fnn}) can be regarded as $h^{[i]}(x)=W_ix$ and $z^{[i]}(x)$'s in (\ref{residual}) can be regarded as $z^{[i]}(x)=W_{i,2}\sigma(W_{i,1}x)+x$.

Next we recall some basics from frame theory. In this paper, we only discuss finite frames for $\mathbb{R}^d$. For the discussion of finite frames for general Hilbert spaces, see \cite{casazza2012finite}. We say a set $\{e_1,\cdots,e_N\}$ in $\mathbb{R}^d$ is a \textit{finite frame} for $\mathbb{R}^d$ if there exist $0<A\le B<\infty$, such that $A\|x\|^2 \le \sum_{i=1}^{N} | \langle x,e_i\rangle|^2 \le B\|x\|^2 $ holds for all $x\in \mathbb{R}^d.$ The constants $A$ and $B$ are called $\textit{frame bounds,}$ and $\{ \langle x,e_i\rangle \}_{i=1}^{N}$ are called the $\textit{frame coefficients}$ of $x$ with respect to frame $\{e_1,\cdots,e_N\}$. A frame $\{e_1,\cdots,e_N\}$ in $\mathbb{R}^d$ is called $\textit{tight}$ if $A=B.$ If a finite tight frame $F=\{e_1,\cdots,e_N\}$ satisfies $\|e_i\|=1$ for all $i$, then we say $F$ is a $\textit{finite unit-norm tight frame}$ (FUNTF). For a finite frame $\{e_1,\cdots,e_N\}$ in $\mathbb{R}^d$, the linear function $S:\mathbb{R}^d \rightarrow \mathbb{R}^d$ defined by
\begin{align}\label{frameoperator}
    Sx=\sum_{i=1}^{N} \langle x,e_i \rangle e_i,
\end{align}
is the \textit{frame operator} of $F$.
Note that if $F=\{e_1,\cdots,e_N\}$ is a finite frame for $\mathbb{R}^d$ with frame bounds $A$ and $B$, then it is easy to see that $S$ is a $d\times d$ positive definite matrix satisfying $AI_d \le S \le BI_d,$ where $I_d$ is the identity matrix of $\mathbb{R}^d$. Therefore, its inverse $S^{-1}$ exists and it satisfies $\frac{1}{B}I_d \le S^{-1} \le \frac{1}{A} I_d.$ This inverse operator $S^{-1}$ is called the \textit{dual frame operator}. 
Multiplying (\ref{frameoperator}) by $S^{-1}$  gives us an atomic decomposition of an arbitrary $x\in\mathbb{R}^d$:
\begin{align}\label{dual expansion}
    x=S^{-1}Sx=S^{-1}\Big\{\sum_{i=1}^{N} \langle x,e_i \rangle e_i\Big\}=\sum_{i=1}^{N} \langle x,e_i \rangle S^{-1}e_i.
\end{align}
It is straightforward to check $\{S^{-1}e_1,\cdots,S^{-1}e_N\}$ is a frame for $\mathbb{R}^d$ with frame bounds $1/B$ and $1/A.$ We say that $\{S^{-1}e_1,\cdots,S^{-1}e_N\}$ is the \textit{canonical dual frame} of $F$. For detailed proofs of the aforementioned statements, see \cite{casazza2012finite}.
When $F=\{e_1,\cdots,e_N\}$ is a tight frame for $\mathbb{R}^d$ with frame bound $A$, then  $S=AI_d$ and $S^{-1}=\frac{1}{A}I_d.$ In this case, (\ref{dual expansion}) gives us the frame expansion $ x=\frac{1}{A}\sum_{i=1}^{N}\langle x,e_i \rangle e_i$, for any $x \in\mathbb{R}^d.$ Moreover, if $F$ is a FUNTF, then it is known that $A=N/d$ holds \cite{benedetto2008complex}. Therefore, we have the frame expansion $x=\frac{d}{N}\sum_{i=1}^{N}\langle x,e_i \rangle e_i$ for any $x \in\mathbb{R}^d.$

\section{Frame quantization}\label{vectorquant}

In vector quantization, first-order $\Sigma\Delta$ quantization algorithm has a uniform upper bound on the error in the case of finite frames \cite{benedetto2006sigma} and it is known to outperform PCM most of the time for realistic settings \cite{benedetto2010pointwise}. Since a matrix can be interpreted as a stack of column vectors, $\Sigma\Delta$ becomes an option for quantizing a weight matrix. In this section, we introduce first-order $\Sigma\Delta$ quantization algorithms for finite frames for $\mathbb{R}^d.$  We begin with some common definitions in \cite{benedetto2006second,benedetto2008complex,casazza2012finite,zhang2023post,benedetto2006sigma}. Given $K\in \mathbb{N}$ and $\delta>0,$ the \textit{midrise quantization alphabet} $A^{\delta}_K$ is defined as
    \begin{align}\label{alphabet}
        A^{\delta}_{K}= \Big\{(-K+\frac{1}{2})\delta,(-K+\frac{3}{2})\delta,\cdots,-\frac{1}{2}\delta,\frac{1}{2}\delta,\cdots,(K-\frac{1}{2})\delta\Big\}.
    \end{align}
    The \textit{$2K$-level midrise uniform scalar quantizer} $Q$ with step size $\delta$ is defined as
    \begin{align*}
        Q(u) = \arg{\min_{q \in A^{\delta}_{K}} |u-q|}.
    \end{align*}
Therefore, $Q(u)$ is the element in $A^{\delta}_{K}$ which is closest to $u$. 
\medskip
\begin{definition}{\normalfont \cite{benedetto2006sigma} }\label{firstorder}
    Given $K\in \mathbb{N}$ and $\delta>0,$ define the midrise quantization alphabet $A^{\delta}_K$ and the $2K$-level midrise uniform scalar quantizer $Q$ with stepsize $\delta$ . Let $F=\{e_1,\cdots,e_N\}$ be a finite frame for $\mathbb{R}^{d}$ where $d\ge3$ and let $p$ be a permutation of $\{1,2,\cdots,N\}.$ For a given input sequence $\{x_1,x_2,\cdots,x_N\}$, the associated first-order \emph{$\Sigma\Delta$ quantization} is defined by the iteration:
    \begin{align*}
        u_n=u_{n-1}+x_{p(n)}-q_n, \notag\\
        q_n=Q(u_{n-1}+x_{p(n)}),
    \end{align*}
    for $n=1,2,\cdots,N$ where $u_0=0.$ This produces the quantized sequence $\{q_1,\cdots,q_N\}$ and an auxiliary sequence $\{u_0,\cdots,u_N\}$ of state variables.
\end{definition}
\medskip
In the above it is possible to consider nonzero initial condition $\|u_0\| < \delta/2,$ but for simplicity, we will only consider the case  $u_0=0$ as in \cite{benedetto2006sigma}. 

We now describe the vector quantization process. Let $F=\{e_1,\cdots,e_N\}$ be a finite frame for $\mathbb{R}^{d}$ and let $p$ be a permutation of $\{1,2,\cdots,N\}.$ Choose an arbitrary vector $x \in \mathbb{R}^d$ and represent it as $ x=\sum_{i=1}^{N}x_iS^{-1}e_i$ with frame expansion in (\ref{dual expansion}). We say $\Bar{x}=\sum_{i=1}^{N}q_iS^{-1}e_{p(i)}$ is the \textit{quantized expansion} of $x$, where $\{q_1,\cdots,q_N\}$ are the quantized sequence from the first-order $\Sigma\Delta$ quantization in Definition \ref{firstorder}.
The first-order $\Sigma\Delta$ scheme in Definition \ref{firstorder} is an iterative scheme that depends heavily on the choice of permutation $p$. Given a  finite frame $F=\{e_1,\cdots,e_N\}$  for $\mathbb{R}^{d}$, its \textit{frame variation} with respect to a permutation $p$ of $\{1,2,\cdots,N\}$ is defined as $\sigma(F,p):=\sum_{i=1}^{N-1}\|e_{p(i)}-e_{p(i+1)}\|.$ Frame variation is a measurement that captures the interdependencies between the frame elements resulting from the choice of $p$. It is an important quantity that reflects the role of permutation $p$ in the error estimates for first-order $\Sigma\Delta$ quantization.

We provide an error bound on the approximation error $\|x-\Bar{x}\|$. Here, we only state the result for the case of a FUNTF. For general results, see \cite{benedetto2006sigma} and \cite{benedetto2008complex}.
\medskip
\begin{theorem}{\normalfont \cite{benedetto2006sigma} }\label{sigmadelta error}
    Let $F=\{e_1,\cdots,e_N\}$ be a finite unit-norm tight frame for $\mathbb{R}^{d}$, and let $p$ be a permutation of $\{1,2,\cdots,N\}.$ Let $x\in\mathbb{R}^d$ satisfy $\|x\| \le (K-1/2)\delta$ and have the frame expansion $x=\sum_{i=1}^{N}\langle x,e_i \rangle S^{-1} e_i,$ where $S^{-1}$ is the inverse frame operator for $F$. Then, $\Bar{x}$ satisfies the approximation error
    \begin{align*}
        \|x-\Bar{x}\| \le \frac{\delta d}{2N}(\sigma(F,p)+1).
    \end{align*}
\end{theorem}
\medskip
To obtain a quantized expansion with a small error, it is desirable to choose a permutation $p$ that makes frame variation $\sigma(F,p)$ small. In \cite{wang2008sigma}, the author proved that for a set $\{e_1,\cdots,e_N\} \subset [-\frac{1}{2},\frac{1}{2}]^{d},$ with $d\ge 3$, there exists a permutation $p$ of $\{1,2,\cdots,N\} $ which satisfies $\sum_{i=1}^{N-1}\|e_{p(i)}-e_{p(i+1)}\|\le 2\sqrt{d+3}N^{1-\frac{1}{d}}-2\sqrt{d+3}.$ Note that for a unit-norm frame $F=\{e_1,\cdots,e_N\}$ for $\mathbb{R}^{d}$, the set $\frac{1}{2}F=\{\frac{1}{2}e_1,\cdots,\frac{1}{2}e_N\}$ is a subset of $[-\frac{1}{2},\frac{1}{2}]^{d}.$ In \cite{benedetto2008complex}, the authors combined these results to obtain the following result.
\medskip
\begin{theorem}{\normalfont \cite{benedetto2008complex}}\label{subopt}
    Let $F=\{e_1,\cdots,e_N\}$ be a unit-norm frame for $\mathbb{R}^{d}$, $d\ge3.$ Then, there exists a permutation $p$ of $\{1,2,\cdots,N\} $ such that
    \begin{align}\label{subopt_bound}
        \sigma(F,p) \le 4\sqrt{d+3}N^{1-\frac{1}{d}}-4\sqrt{d+3}.
    \end{align}
\end{theorem}
\medskip

Next, we estimate the approximation error in terms of $\delta,d,$ and $N$.
\medskip
\begin{corollary}\label{besterror}
    Let $F=\{e_1,\cdots,e_N\}$ be a finite unit-norm tight frame for $\mathbb{R}^{d}$, $d\ge 3,$ and let $x\in\mathbb{R}^d$ satisfy $\|x\| \le (K-1/2)\delta$. Let $p$ be a permutation of $\{1,2,\cdots,N\}$ that satisfies (\ref{subopt_bound}). Let $\Bar{x}$ be the quantized expansion. Then, the approximation error satisfies
    \begin{align*}
        \|x-\Bar{x}\| \le \frac{\delta d}{2N} (4\sqrt{d+3}N^{1-\frac{1}{d}}-4\sqrt{d+3} +1).
    \end{align*}
\end{corollary}

\begin{proof}
    Apply Theorem \ref{subopt} to Theorem \ref{sigmadelta error}.
\end{proof}

Corollary \ref{besterror} shows that we can bound the approximation error by only using the variables $\delta,d,$ and $N$. We use this result to obtain error bounds between a neural network and its quantized version in Section 5.

\section{Frame Quantization for Neural Networks}\label{ouralg}
In this section, we explain our method to quantize a weight matrix $W_i \in \mathbb{R}^{m_i \times m_{i-1}}$. Let 
$F_i=\{e^i_1,\cdots,e^i_{N_i}\} $ be a FUNTF for $ \mathbb{R}^{m_i}$ that we are using for quantization. Write $W_i=[w^i_1, w^i_2, \cdots, w^i_{m_{i-1}}]$ where $ w^i_j$ is a $m_i\times 1$ column vector for $j=1,2,\cdots,m_{i-1}.$ First, choose appropriate positive integer $K_i$ and step size $\delta_i>0$ that satisfy
\begin{align}\label{howtochoosedelta} 
    \max_{j=1,\cdots,m_{i-1}} \|w^i_j\| \le (K_i-\frac{1}{2})\delta_i.
\end{align}
Next, we find a permutation $p_i$ of $\{1,\cdots,N_i\}$ that satisfies (\ref{subopt_bound}). The algorithm to find such permutations is described in \cite{benedetto2008complex,wang2008sigma}, so we omit the details. Then, we compute the quantized expansions of $w^i_j$'s, using frame $F_i$, constant $K_i,$ and step size $\delta_i$ for $i=1,\cdots,n$. Let $q^i_j$ be the quantized expansion of $w^i_j.$ Since we are using finite unit-norm tight frame, the dual frame operator $S^{-1}$ is $\frac{m_i}{N_i}I_{m_i}$, where $I_{m_i}$ is the identity matrix of $ \mathbb{R}^{m_i}$. Therefore $q^i_j$'s can be written as
\begin{align}\label{reconstruct}
    q^{i}_{j}=\sum_{k=1}^{N_i} q^i_{j,k}S^{-1}e^i_{p_i(k)}=\frac{m_i}{N_i}\sum_{k=1}^{N_i}q^{i}_{j,k}e^{i}_{p_i(k)}.
\end{align}
Note that the first-order $\Sigma\Delta$ quantization in Definition \ref{firstorder} forces $q^i_{j,k} \in A^{\delta_i}_{K_i}$ for all $(j,k)\in\{1,\cdots,m_{i-1}\}\times\{1,\cdots,N_i\}.$ Therefore, at most $2K_i$ values are candidates for $q^i_{j,k}$. In our scenario, we store $C_i=[q^i_{j,k}]_{1\le j \le m_{i-1},1\le k \le N_i} \in  (A^{\delta_i}_{K_i})^{m_{i-1}\times N_i}$. When we need a quantized neural network, then we use the finite unit-norm tight frames $F_i$'s and matrices $C_i$'s to reconstruct a network using (\ref{reconstruct}).

Current neural network quantization methods convert a weight matrix $W=[w_{ij}]_{1\le i \le m,1\le j \le n} \in \mathbb{R}^{m\times n}$ into a quantized matrix $Q=[q_{ij}]_{1\le i \le m,1\le j \le n}$,  where each $q_{ij}$ uses fewer bits compared to the corresponding $w_{ij}$. Note that $W$ and $Q$ have the same dimensions. Our method differs from typical approaches because we are storing matrices $C_i$'s that have different dimensions from the weight matrices $W_i$'s. Due to their different dimensions, we will not say $C_i$'s are quantized weights. Instead, we will say that the matrix $Q_i=[q^i_1, \cdots, q^i_{m_{i-1}}]\in \mathbb{R}^{m_i \times m_{i-1}}$, where the columns are the quantized expressions $q^{i}_{j}$'s, is the \textit{quantized weight matrix} of $W_i\in \mathbb{R}^{m_i \times m_{i-1}}$ for $i=1,\cdots,n.$ 

\begin{algorithm}
\caption{Frame Quantization}
\label{framequant}
\begin{algorithmic}
\Require Weight matrix $W_i \in \mathbb{R}^{m_i \times m_{i-1}}$ and FUNTF $F_i=\{e^i_1,\cdots,e^i_{N_i}\} $ for $ \mathbb{R}^{m_i}.$
\State 1. Set level $K_i\in\mathbb{N}$ and stepsize $\delta_i$ that satisfy $\max_{j=1,\cdots,m_{i-1}} \|w^i_j\| \le (K_i-\frac{1}{2})\delta_i.$
\State 2. Find a permutation $p_i$ of $\{1,\cdots,N_i\}$ that satisfies (\ref{subopt_bound}).
\For{$j=1,\cdots,m_{i-1}$}
    \State 1. Compute frame expansion $w^{i}_{j}=\frac{m_i}{N_i}\sum_{k=1}^{N_i}w^{i}_{j,k}e^{i}_{k}.$
    \State 2. Use first-order $\Sigma\Delta$ quantization algorithm \ref{firstorder} with $K_i$ and $\delta_i$ to compute 
    $$ q^{i}_{j}=\frac{m_i}{N_i}\sum_{k=1}^{N_i}q^{i}_{j,k}e^{i}_{p_i(k)}.$$
\EndFor
\State 3. Set $C_i=[q^i_{j,k}]_{1\le j \le m_{i-1},1\le k \le N_i} \in  (A^{\delta_i}_{K_i})^{m_{i-1}\times N_i}.$
\Ensure Quantized matrix $Q_i=[q^i_1, \cdots, q^i_{m_{i-1}}]\in \mathbb{R}^{m_i \times m_{i-1}}. $ Store $C_i.$

\end{algorithmic}
\end{algorithm}

Note that the above algorithm uses first-order $\Sigma\Delta$ quantization for the column vectors of the weight matrix $W_i.$ On the other hand, we may consider applying first-order $\Sigma\Delta$ quantization algorithm for the row vectors of the weight matrix $W_i$. In this case, we apply Algorithm \ref{framequant} to $W_i^T$.

Note that we can also quantize biases by adding them to the final columns of weight matrices. This means we can also quantize $b_i$'s while we quantize the weight matrices $W_i$'s by applying Algorithm \ref{framequant} to matrices $\widetilde{W_i}=(W_i,b_i)$'s.

\section{Error Estimates}\label{errorestimates}
We begin with some auxiliary lemmas, which play an important role in proving error estimates between a neural network and its approximation by a quantized network.
\medskip
\begin{lemma}\label{matrixerrorbound}
    Let $Q_i$ be the quantized matrix for weight matrix $W_i \in \mathbb{R}^{m_i\times m_{i-1}}$ using Algorithm \ref{framequant}. Then, 
    \begin{align*}
    \|W_i-Q_i\| \le 2\sqrt{2}\delta_i m_i\sqrt{m_{i-1}m_i}N_i^{-\frac{1}{m_i}},
    \end{align*}
    where $\delta_i$ is the step size and $N_i$ is the number of elements in the frame $F_i=\{e^i_1,\cdots,e^i_{N_i}\}$ used in Algorithm \ref{framequant}.   
\end{lemma}
\begin{proof}
    Write $W_i=[w^i_1 \cdots w^{i}_{m_{i-1}}], Q_i=[q^i_1, \cdots, q^i_{m_{i-1}}]\in \mathbb{R}^{m_i \times m_{i-1}} $ where $w^i_j$'s and $q^i_j$'s are $m_i\times 1$ column vectors. Since Algorithm \ref{framequant} uses $p_i$ that satisfies (\ref{subopt_bound}), Corollary \ref{besterror} gives us the error estimate
\begin{align}\label{columnerror}
    \|w^i_j-q^i_j\| \le \frac{\delta_i m_i}{2N_i}(4\sqrt{m_i+3}N_i^{1-\frac{1}{m_i}}-4\sqrt{m_i+3}+1)
\end{align}
for all $j=1,\cdots,m_{i-1}.$ Now, using  (\ref{columnerror}) and the assumption $m_i \ge 3$ from Section \ref{Prelim}, we have
\begin{align*}
    \|W_i-Q_i\|&=\max_{\|x\|=1}\|(W_i-Q_i)x\| \le \max_{\|x\|=1}\{\sum_{j=1}^{m_{i-1}}|x_j|\times\|w^i_j-q^i_j\|\}\notag\\
    &\le \frac{\delta_i m_i}{2N_i}(4\sqrt{m_i+3}N_i^{1-\frac{1}{m_i}}-4\sqrt{m_i+3}+1) \times \max_{\|x\|=1}\{\sum_{j=1}^{m_{i-1}}|x_j|\} \notag\\
    & \le \frac{\delta_i m_i\sqrt{m_{i-1}}}{2N_i}(4\sqrt{m_i+3}N_i^{1-\frac{1}{m_i}}-4\sqrt{m_i+3}+1) \notag\\
    &\le \frac{\delta_i m_i\sqrt{m_{i-1}}}{2N_i}\times4\sqrt{m_i+3}N_i^{1-\frac{1}{m_i}} \le 2\sqrt{2}\delta_i m_i\sqrt{m_{i-1}m_i}N_i^{-\frac{1}{m_i}}.
\end{align*}
\end{proof}

Next we provide an upper bound for the norm of the quantized matrices $Q_i$'s. This is used in deriving the error estimate between an $n$-layer FNN and its corresponding quantized neural network.

\medskip
\begin{lemma}\label{quantized_norm}
    Let $Q_i$ be the quantized matrix of $W_i.$ Let $\sigma_i=\|W_i\|.$ Then we have
    \begin{align*}
        \|Q_i\| \le 2\sqrt{2}\delta_j m_j\sqrt{m_j m_{j-1}}N_j^{-\frac{1}{m_j}}+\sigma_i.
    \end{align*}
\end{lemma}
\begin{proof}
    Using Lemma \ref{matrixerrorbound} and the triangle inequality of matrix norms, we have
    \begin{align*}
        \|Q_i\| \le \|Q_i-W_i\|+\|W_i\| \le 2\sqrt{2}\delta_i m_i\sqrt{m_i m_{i-1}}N_{i}^{-\frac{1}{m_i}}+\sigma_i.
    \end{align*}
\end{proof}

\subsection{Feedforward Networks}

In this section, we derive an upper estimate for $\|f(x)-f_Q(x)\|$, where $f(x)$ is a FNN with $n$ layers (\ref{fnn}) and $f_Q(x)$ its quantized neural network (\ref{quant_fnn}). For convenience, we adopt the assumptions from Section \ref{Prelim} and omit the biases in our analysis. Thus, we can write $f(x)=W_n(\sigma(\cdots\sigma(W_{1}x)\cdots))$ and $f_Q(x)=Q_n(\sigma(\cdots\sigma(Q_{1}x)\cdots))$. As in Lemma \ref{quantized_norm}, we let $\sigma_i=\|W_i\|$ for $i=1,\cdots,n.$

\medskip
\begin{theorem}\label{maintheorem}
    Let $f(x)=W_n(\sigma(\cdots\sigma(W_{1}x)\cdots))$ be a feed-forward neural network with $n$ layers. Let $f_Q(x)=Q_n(\sigma(\cdots\sigma(Q_{1}x)\cdots))$ be the quantized neural network obtained by means of Algorithm \ref{framequant}. Then, for any input $X\in\mathbb{R}^{m_0}$, we have
    \begin{align}\label{error_fnn}
        \|f(X)-f_Q(X)\| &\le L^{n-1}\|X\|\times\sum_{j=1}^{n}\Big\{2\sqrt{2}\delta_j m_j\sqrt{m_j m_{j-1}}N_{j}^{-\frac{1}{m_j}} \times \prod_{i=j+1}^{n}\sigma_i \notag\\
        &\times \prod_{l=1}^{j-1} (2\sqrt{2}\delta_l m_l\sqrt{m_l m_{l-1}}N_{l}^{-\frac{1}{m_l}}+\sigma_l)\Big\}.
    \end{align}
\end{theorem}

\begin{proof}
    Note that since $\sigma$ is an $L-$Lipschitz continuous function with $\sigma(0)=0,$ we have
    \begin{align*}
        \|\sigma(v)\| = \|\sigma(v)-\sigma(0)\| \le L\|v-0\|=L\|v\|
    \end{align*}
    for any vector $v$. Using this, for $j=1,\cdots,n,$ we have
    \begin{gather*}
        \|W_n\sigma\cdots W_{j}\sigma Q_{j-1}\sigma Q_{j-2}\cdots Q_1X-W_n\sigma\cdots W_{j+1}\sigma Q_{j}\sigma Q_{j-1}\sigma Q_{j-2}\cdots Q_1X\|
    \\
    \le L^{n-j} \|W_n\|\cdots \|W_{j+1}\| \times \|W_j-Q_j\| \|\sigma Q_{j-1}\sigma Q_{j-2}\cdots Q_1X\|
    \\
    \le L^{n-1}  \|W_n\|\cdots \|W_{j+1}\| \times \|W_j-Q_j\| \times \|Q_{j-1}\|\cdots \|Q_1\| \|X\| .
    \end{gather*}
Next, using Lemma \ref{matrixerrorbound}, Lemma \ref{quantized_norm}, and triangle inequality, we get
\begin{gather*}
    \|f(X)-f_Q(X)\|=\|W_n(\sigma(\cdots\sigma(W_{1}x)\cdots))-Q_n(\sigma(\cdots\sigma(Q_{1}x)\cdots))\| 
    \notag\\
     \le \sum_{j=1}^{n} \|W_n\sigma\cdots W_{j}\sigma Q_{j-1}\sigma Q_{j-2}\cdots Q_1X-W_n\sigma\cdots W_{j+1}\sigma Q_{j}\sigma Q_{j-1}\sigma Q_{j-2}\cdots Q_1X\|
    \notag\\
     \le \sum_{j=1}^{n} L^{n-1}  \|W_n\|\cdots \|W_{j+1}\| \times \|W_j-Q_j\| \times \|Q_{j-1}\|\cdots \|Q_1\| \|X\|
    \notag\\
     \le L^{n-1}\|X\|\times\sum_{j=1}^{n}\Big\{2\sqrt{2}\delta_j m_j\sqrt{m_j m_{j-1}}N_{j}^{-\frac{1}{m_j}} \times \prod_{i=j+1}^{n}\sigma_i \notag \\
     \times \prod_{l=1}^{j-1} (2\sqrt{2}\delta_l m_l\sqrt{m_l m_{l-1}}N_{l}^{-\frac{1}{m_l}}+\sigma_l)\Big\}.
\end{gather*}
\end{proof}

Note that $\{\sigma_i\}_{i=1}^{n}, \{m_i\}_{i=0}^{n},$ and $L$ are constants determined by the given FNN $f(x).$ Therefore, the upper bound in (\ref{error_fnn}) can be only controlled by the step sizes $\delta_i$ and the numbers of frame elements $N_i$. Once we fix $\delta_i$'s, then the upper bound in (\ref{error_fnn}) can be written as $\sum_{j=1}^{n} O(N_j^{-1/m_j})\|X\|.$ This shows we can leverage the redundancy of frames for the accuracy of quantized neural networks. On the other hand, if we fix frames $F_1,\cdots,F_n$, then the upper bound in (\ref{error_fnn}) depends only on $\delta_i$'s and can be written as $\sum_{i=1}^{n} O(\delta_i)\|X\|.$ Therefore, smaller $\delta_i$'s and larger $N_i$'s would generate a quantized neural network with higher accuracy. 

When the number of neurons in each hidden layer is the same, that is when $m_1=\cdots=m_{n-1}$, then we have a simplified version of Theorem \ref{maintheorem}. 
\medskip
\begin{corollary} \label{samehidden}
    Let $f(x)=W_n(\sigma(\cdots\sigma(W_{1}x)\cdots))$ be a feed-forward neural network with $n$ layer with $m_1=m_2=\cdots=m_{n-1}=m$. Fix a FUNTF $F \subset \mathbb{R}^m$ with $N$ elements and a stepsize $\delta.$ Choose a permutation $p$ of $\{1,2,\cdots,N\}$ that satisfies (\ref{subopt_bound}). Now, use Algorithm \ref{framequant} to quantize $W_1,\cdots,W_{n-1}$ and $W_n^T$ with the same $F,\delta,$ and $p$. Let $M=\max\{m_0,m,m_n\}.$ Then, the quantized neural network $f_Q(x)=Q_n(\sigma(\cdots\sigma(Q_{1}x)\cdots))$ satisfies
    \begin{align}\label{samehiddenineq}
        \|f(X)-f_Q(X)\| \le 2\sqrt{2}\delta M^2 N^{-\frac{1}{m}}L^{n-1}\|X\|\sum_{j=1}^{n}\Big\{\prod_{i=j+1}^{n}\sigma_i\prod_{l=1}^{j-1} (2\sqrt{2}\delta M^2 N^{-\frac{1}{m}}+\sigma_l)\Big\}
    \end{align}
    for any data $X\in \mathbb{R}^{m_0}.$
\end{corollary}

\begin{proof}
    By Lemma \ref{matrixerrorbound}, we have $\|W_i-Q_i\|\le 2\sqrt{2}\delta_i m_i\sqrt{m_{i-1}m_i}N_i^{-\frac{1}{m_i}} \le 2\sqrt{2}\delta M^2 N^{-\frac{1}{m}}$ for $i=1,\cdots,n-1.$ For $i=n,$ we have $\|W_n-Q_n\| = \|W_n^T-Q_n^T\| \le 2\sqrt{2}\delta_n m_{n-1}\sqrt{m_{n-1}m_n}N_n^{-\frac{1}{m_{n-1}}} \le 2\sqrt{2}\delta M^2 N^{-\frac{1}{m}}$. These estimates yield $\|Q_i\| \le \|W_i-Q_i\|+\|W_i\| \le 2\sqrt{2}\delta M^2 N^{-\frac{1}{m}}+\sigma_i.$ Therefore, from the proof of Theorem \ref{maintheorem}, we get
    \begin{gather*}
        \|f(X)-f_Q(X)\| \le L^{n-1}\|X\|\times\sum_{j=1}^{n}\Big\{\prod_{i=j+1}^{n}\sigma_i\times\|W_j-Q_j\|\times\|Q_{j-1}\|\cdots\|Q_1\|\Big\} \\
        \le  L^{n-1}\|X\|\sum_{j=1}^{n}\Big\{2\sqrt{2}\delta M^2 N^{-\frac{1}{m}}\prod_{i=j+1}^{n}\sigma_i\prod_{l=1}^{j-1} (2\sqrt{2}\delta M^2 N^{-\frac{1}{m}}+\sigma_l)\Big\} \\
        =2\sqrt{2}\delta M^2 N^{-\frac{1}{m}}L^{n-1}\|X\|\sum_{j=1}^{n}\Big\{\prod_{i=j+1}^{n}\sigma_i\prod_{l=1}^{j-1} (2\sqrt{2}\delta M^2 N^{-\frac{1}{m}}+\sigma_l)\Big\}.
    \end{gather*}
\end{proof}
If we use a FUNTF for $\mathbb{R}^m$ with $N$ elements where $N\ge \Big( \frac{2\sqrt{2}\delta M^2  }{\min\{\sigma_1,\cdots,\sigma_n\}}\Big)^{m},$ then we have the following simplification.
\medskip
\begin{corollary}
    If we have $N\ge \Big( \frac{2\sqrt{2}\delta M^2  }{\min\{\sigma_1,\cdots,\sigma_n\}}\Big)^{m}$ in Corollary \ref{samehidden}, then
    \begin{align*}
        \|f(X)-f_Q(X)\| \le \sqrt{2}\delta M^2 N^{-\frac{1}{m}}L^{n-1}\|X\|\prod_{i=1}^{n}\sigma_i\sum_{j=1}^{n}\frac{2^j}{\sigma_j},
    \end{align*}
    for any input data $X\in \mathbb{R}^{m_0}.$
\end{corollary}

\begin{proof}
    $N\ge \Big( \frac{2\sqrt{2}\delta M^2  }{\min\{\sigma_1,\cdots,\sigma_n\}}\Big)^{m}$ implies $2\sqrt{2}\delta m^2 N^{-\frac{1}{m}} \le \min\{\sigma_1,\cdots,\sigma_n\}.$ Hence we have
    \begin{align*}
        \prod_{i=j+1}^{n}\sigma_i\prod_{l=1}^{j-1} (2\sqrt{2}\delta M^2 N^{-\frac{1}{m}}+\sigma_l)  \le \prod_{i=j+1}^{n}\sigma_i\prod_{l=1}^{j-1}(\sigma_l+\sigma_l)=\frac{2^{j-1}}{\sigma_j}\prod_{i=1}^{n}\sigma_i.
    \end{align*}
    Applying this to (\ref{samehiddenineq}) yields the result.
\end{proof}
\medskip
Note that the above results are for general FUNTFs. On the other hand, some special families of FUNTFs yield better error estimates. For example, harmonic frame $H^d_N$ is a FUNTF for $\mathbb{R}^d$ with $N$ elements that satisfy $\sigma(H^d_N,p)\le\frac{2\pi(d+1)}{\sqrt{3}}$ for identity permutation $p$. This holds for any $N\ge d$, see \cite{benedetto2006sigma} for details. If we use harmonic frames $\{H^{m_i}_{N_i}\}_{i=1}^{n}$ and identity permutation to quantize weight matrices $W_i$'s using Algorithm \ref{framequant}, then one can easily prove a variant of Lemma \ref{matrixerrorbound}:
    $$ \|W_i-Q_i\| \le \frac{\delta_i m_i\sqrt{m_{i-1}}}{2N_i}\Big(\frac{2\pi(m_i+1)}{\sqrt{3}}+1\Big) \le \frac{\delta_i m_i\sqrt{m_{i-1}}}{2N_i}\times \frac{8\pi+\sqrt{3}}{3\sqrt{3}}m_i.$$
The last inequality holds since we assume $m_i\ge 3$ for all $i$'s. Using this, we can show (\ref{error_fnn}) can be rewritten as an explicit upper bound for harmonic frames, which is a stronger upper bound in terms of $N_i$'s:
\begin{align}\label{harmonicerror}
    \|f(X)-f_Q(X)\| &\le L^{n-1}\|X\|\times\sum_{j=1}^{n}\Big\{\frac{(8\pi+\sqrt{3})}{6\sqrt{3}}\delta_j m_j\sqrt{m_j m_{j-1}}N_{j}^{-1} \times \prod_{i=j+1}^{n}\sigma_i \notag\\
        &\times \prod_{l=1}^{j-1} (\frac{(8\pi+\sqrt{3})}{6\sqrt{3}}\delta_l m_l\sqrt{m_l m_{l-1}}N_{l}^{-1}+\sigma_l)\Big\}.
\end{align}
That is, if we treat $\delta_i$'s as fixed constants, then we can say $\|f(X)-f_Q(X)\| \le \sum_{i=1}^{n} O(N_i^{-1}) \|X\|.$ We can get similar results when we use families of frames such as harmonic frames, where their frame variations are uniformly bounded. These frames may help us to construct a quantized network with higher accuracy.

\subsection{Neural Networks with Residual Blocks}
We now prove error estimates for residual networks. Again, for convenience, we follow the assumptions from Section \ref{Prelim} and omit the biases in our analysis. Recall that from (\ref{residual}), the weight matrices $W_{i,1}$'s and $W_{i,2}$'s are all $k\times k$ square matrices. Hence we quantize all the weight matrices using the same frame $F$ with $N$ elements, the same permutation $p$ of $\{1,2,\cdots,N\}$ that satisfies (\ref{subopt_bound}), the same constant $K$, and the same step size $\delta$ as in Algorithm \ref{framequant}. Note that the quantized weight matrices $Q_{i,1}$'s and $Q_{i,2}$'s are also $k\times k$ matrices.
\medskip
\begin{theorem}\label{main_rnn}
    Let $g(x)$ be a residual neural network with $n$ residual blocks (\ref{residual}). Let $$\lambda=\max_{i=1,\cdots,n} \Big\{\max\{\|W_{i,2}\|,\|W_{i,1}\|\}\Big\}.$$
    Let $g_Q(x)$ be the quantized residual neural network of $g(x)$ (\ref{quant_res}) from Algorithm \ref{framequant}. Then for any input $X \in \mathbb{R}^k,$ we have
    \begin{align}\label{rnnbound}
        \|g(X)-g_Q(X)\| &\le 4\delta k\sqrt{k(k+3)}(\delta k\sqrt{k(k+3)}+\lambda)N^{-\frac{1}{k}}\|X\| \notag \\
        &\times\sum_{j=0}^{n-1}(\lambda^2+1)^{j}\Big((2\delta k\sqrt{k(k+3)}N^{-\frac{1}{k}}+\lambda)^2+1\Big)^{n-1-j}.
    \end{align}
\end{theorem}
\begin{proof}
    For simplicity, let $A=4\delta k\sqrt{k(k+3)}(\delta k\sqrt{k(k+3)}+\lambda)N^{-\frac{1}{k}}$ and $B=(2\delta k\sqrt{k(k+3)}N^{-\frac{1}{k}}+\lambda)^2+1$. Let $y_0(x)=y_{0,Q}(x)=x$ and define $y_i(x)=z^{[i]}\circ \sigma  \cdots \circ \sigma \circ z^{[1]}(x)$, and $y_{i,Q}(x)=z_{Q}^{[i]}\circ \sigma  \cdots \circ \sigma \circ z_{Q}^{[1]}(x)$ for $i=1,\cdots,n$.  We shall show that 
    \begin{align}\label{induction}
        \|y_i(X)-y_{i,Q}(X)\| \le A\|X\|\sum_{j=0}^{i-1} (\lambda^2+1)^j B^{i-1-j},
    \end{align}
    for $i=1,\cdots,n.$ Note that we have recurrence relations:
    \begin{align}
        y_i(X)=W_{i,2}(\sigma(W_{i,1}(\sigma (y_{i-1}(X)))))+\sigma (y_{i-1}(X)), \label{rec1} \\
         y_{i,Q}(X)=Q_{i,2}(\sigma(Q_{i,1}(\sigma (y_{i-1,Q}(X)))))+\sigma(y_{i-1,Q}(X)). \label{rec2}
    \end{align}
    
    By using the triangle inequality repeatedly, we obtain:
    \begin{align}\label{rec_tri}
    \|y_i(X)-y_{i,Q}(X)\| & \le \|W_{i,2}(\sigma(W_{i,1}(\sigma (y_{i-1}(X)))))-Q_{i,2}(\sigma(Q_{i,1}(\sigma (y_{i-1,Q}(X)))))\| \notag\\
    &+\|\sigma (y_{i-1}(X))-\sigma (y_{i-1,Q}(X))\| \notag \\
    &\le \|W_{i,2}(\sigma(W_{i,1}(\sigma( y_{i-1}(X)))))-W_{i,2}(\sigma(Q_{i,1}(\sigma (y_{i-1}(X)))))\| \notag \\
    &+\|W_{i,2}(\sigma(Q_{i,1}(\sigma (y_{i-1}(X)))))-Q_{i,2}(\sigma(Q_{i,1}(\sigma (y_{i-1}(X)))))\| \notag \\
    &+\|Q_{i,2}(\sigma(Q_{i,1}(\sigma (y_{i-1}(X)))))-Q_{i,2}(\sigma(Q_{i,1}(\sigma (y_{i-1,Q}(X)))))\| \notag\\
    &+\|\sigma (y_{i-1}(X))-\sigma (y_{i-1,Q}(X))\|.
    \end{align}

Since $\sigma$ is ReLU, we have $\|\sigma (x)-\sigma (y)\| \le \|x-y\|$ and $\|\sigma(x)\|\le\|x\|$ for any vectors $x,y\in \mathbb{R}^k.$ Using this fact, the first term in (\ref{rec_tri}) is bounded by:
\begin{gather*}
    \|W_{i,2}(\sigma(W_{i,1}(\sigma( y_{i-1}(X)))))-W_{i,2}(\sigma(Q_{i,1}(\sigma (y_{i-1}(X)))))\| \\
    \le \|W_{i,2}\|\|W_{i,1}-Q_{i,1}\|\|y_{i-1}(X)\| \le \lambda \|W_{i,1}-Q_{i,1}\|\|y_{i-1}(X)\|.
\end{gather*}
The second term in (\ref{rec_tri}) is bounded by:
\begin{gather*}
    \|W_{i,2}(\sigma(Q_{i,1}(\sigma (y_{i-1}(X)))))-Q_{i,2}(\sigma(Q_{i,1}(\sigma (y_{i-1}(X)))))\| \\
    \le \|W_{i,2}-Q_{i,2}\|\|\sigma(Q_{i,1}(\sigma (y_{i-1}(X)))) \| \le \|W_{i,2}-Q_{i,2}\|\|Q_{i,1}\| \|y_{i-1}(X)\|.
\end{gather*}
The third term in (\ref{rec_tri}) is bounded by:
\begin{gather*}
    \|Q_{i,2}(\sigma(Q_{i,1}(\sigma (y_{i-1}(X)))))-Q_{i,2}(\sigma(Q_{i,1}(\sigma (y_{i-1,Q}(X)))))\|\\
    \le \|Q_{i,2}\|\|Q_{i,1}\| \| y_{i-1}(X) - y_{i-1,Q}(X) \|.
\end{gather*}
Finally, the fourth term in (\ref{rec_tri}) is bounded by:
\begin{gather*}
    \|\sigma (y_{i-1}(X))-\sigma (y_{i-1,Q}(X))\| \le \| y_{i-1}(X)-y_{i-1,Q}(X)\|.
\end{gather*}
From the four inequalities above, we obtain:
\begin{align}\label{ineq5}
    \|y_i(X)-y_{i,Q}(X)\| &\le \{\lambda \|W_{i,1}-Q_{i,1}\|+\|W_{i,2}-Q_{i,2}\|\|Q_{i,1}\|\}\times\|y_{i-1}(X)\| \notag \\
    &+ \{\|Q_{i,2}\|\|Q_{i,1}\|+1\}\times\| y_{i-1}(X)-y_{i-1,Q}(X)\|.
\end{align}

From Lemma \ref{matrixerrorbound} and Lemma \ref{quantized_norm}, we have $\|W_{i,1}-Q_{i,1}\|,\|W_{i,2}-Q_{i,2}\| \le 2\delta k\sqrt{k(k+3)}N^{-\frac{1}{k}}$ and $\|Q_{i,1}\|,\|Q_{i,2}\| \le 2\delta k\sqrt{k(k+3)}N^{-\frac{1}{k}}+\lambda.$ Applying these to (\ref{ineq5}), we have:
\begin{align}\label{ineq6}
    \|y_i(X)-y_{i,Q}(X)\| &\le 2\delta k\sqrt{k(k+3)}N^{-\frac{1}{k}}(\lambda+2\delta k\sqrt{k(k+3)}N^{-\frac{1}{k}}+\lambda)\|y_{i-1}(X)\| \notag \\
    &+ \Big((2\delta k\sqrt{k(k+3)}N^{-\frac{1}{k}}+\lambda)^2+1\Big)\| y_{i-1}(X)-y_{i-1,Q}(X)\| \notag \\
    & = A\|y_{i-1}(X)\|+B\| y_{i-1}(X)-y_{i-1,Q}(X)\|. 
\end{align}

  From the recurrence relations (\ref{rec1}) and (\ref{rec2}), one can easily check that
    \begin{align}\label{repeat}
        \|y_i(X)\| &= \| W_{i,2}(\sigma(W_{i,1}(\sigma (y_{i-1}(X)))))+\sigma (y_{i-1}(X))\| \notag \\
        & \le \| W_{i,2}\| \| W_{i,1}\| \| \| y_{i-1}(X)\| + \| y_{i-1}(X)\| \le (\lambda^2+1)\| y_{i-1}(X)\|.
    \end{align}
    
    Using (\ref{repeat}) repeatedly, we obtain $\|y_i(X)\| \le (\lambda^2+1)^i\| y_0(X)\|= (\lambda^2+1)^i\|X\|,$ which together with (\ref{ineq6}) implies:
    \begin{align}\label{ineq7}
        \|y_i(X)-y_{i,Q}(X)\| & \le A\|y_{i-1}(X)\|+B\| y_{i-1}(X)-y_{i-1,Q}(X)\| \notag \\
        & \le A(\lambda^2+1)^{i-1}\|X\|+B\| y_{i-1}(X)-y_{i-1,Q}(X)\| .
    \end{align}

    To prove (\ref{induction}), we proceed by induction. When $i=1,$ (\ref{ineq7}) yields: $$\|y_1(X)-y_{1,Q}(X)\| \le A(\lambda^2+1)^0\|X\|+B\|y_0(X)-y_{0,Q}(X)\|=A\|X\|.$$
    
    Therefore, (\ref{induction}) holds when $i=1.$ Now, let's assume that (\ref{induction}) holds for $i=t.$ When $i=t+1,$ by (\ref{ineq7}) and induction hypothesis, we have
    \begin{gather*}
        \|y_{t+1}(X)-y_{t+1,Q}(X)\| \le A(\lambda^2+1)^t\|X\|+B\|y_t(X)-y_{t,Q}(X)\| \\
        \le A\|X\| \{ (\lambda^2+1)^t+B\sum_{j=0}^{t-1} (\lambda^2+1)^j B^{t-1-j} \} = A\|X\|\sum_{j=0}^{t} (\lambda^2+1)^j B^{t-j}.
    \end{gather*}
    Hence, (\ref{induction}) holds. Note that $g(x)=y_n(x)$ and $g_Q(x)=y_{n,Q}(x).$ Taking $i=n$ in (\ref{induction}) completes the proof.
\end{proof}

Theorem \ref{main_rnn} shows that using smaller step size $\delta$ and larger $N$ can generate a quantized neural network with higher accuracy. If we treat $\delta$ as a fixed constant, we can view the upper bound in (\ref{rnnbound}) as $O(N^{-1/k})\|X\|$. This is the result for general FUNTFs. If we use a family of frames with uniformly bounded variation like harmonic frames, then one can also prove that the upper bound in (\ref{rnnbound}) is $O(N^{-1})\|X\|$ if we treat $\delta$ as a fixed constant.

\section{Numerical Results}\label{numresults}
In this section, we present numerical results which demonstrate the performance of our proposed algorithm.  We analyze the accuracy of frame quantized networks for a classification task with the MNIST dataset. For our implementation, we used the same $K$ and step size $\delta$ for quantizing each layer as well as the residual blocks. The nonlinear activation function $\sigma$ is ReLU. We used harmonic frames $H^d_N$ and identity permutation for Algorithm \ref{framequant}. We trained each neural network 10 times and quantized it with different values of $N$ and $\delta,$ where $N$ is the number of elements in harmonic frames. We report the average accuracy and standard deviations.

\subsection{Feedforward Network with 3 Layers}
FNN in our experiment is $f(x)=h^{[3]}\circ \sigma \circ h^{[2]} \circ \sigma \circ h^{[1]} (x)$ where $h^{[1]}:\mathbb{R}^{784}\rightarrow\mathbb{R}^{256}, h^{[2]}:\mathbb{R}^{256}\rightarrow\mathbb{R}^{256},$ and $h^{[3]}:\mathbb{R}^{256}\rightarrow\mathbb{R}^{10}$ are affine maps in as in (\ref{fnn}) without bias terms.  Our results are presented in Table \ref{table:1}. As expected, quantizing with smaller $\delta$'s and larger $N$'s results in higher accuracy.

\begin{table}[ht]
\centering
\begin{tabular}{|c| c| c| c|c|c|} 
 \hline
 $N$ & $\delta=1/16$ & $\delta=1/8$ & $\delta=1/4$ & $\delta=1/2$ & $\delta=1$  \\ 
 \hline
 256 & $97.53\pm0.24\%$ & $97.03\pm0.35\%$& $93.39\pm1.92\%$ & $63.76\pm4.54\%$ & $22.85\pm3.67\%$ \\ 
 \hline
 320 & $97.62\pm0.19\%$ & $97.35\pm0.29\%$& $95.47\pm1.34\%$ & $84.93\pm4.66\%$ & $50.30\pm8.24\%$ \\
 \hline
 384 & $97.65\pm0.25\%$ & $97.46\pm0.28\%$& $95.99\pm1.59\%$ & $90.68\pm3.80\%$ & $62.48\pm6.54\%$ \\
 \hline
 448 & $97.68\pm0.22\%$ & $97.55\pm0.24\%$& $96.92\pm0.59\%$ & $93.75\pm2.03\%$ & $76.66\pm5.25\%$ \\
 \hline
 512 & $97.68\pm0.25\%$ & $97.57\pm0.25\%$& $97.20\pm0.30\%$ & $95.71\pm1.36\%$ & $86.97\pm1.92\%$ \\
 \hline
\end{tabular}
\caption{Frame Quantization for FNN with 3 layers. The accuracy for the trained neural network is $97.72\pm 0.21\%$.}
\label{table:1}
\end{table}

We now demonstrate the connection between theory and numerical results. Since we are using harmonic frames with the same $N$ and $\delta$ for quantizing each layer, the error bound (\ref{harmonicerror}) can be represented as $\|f(X)-f_Q(X)\| \le C\delta N^{-1}\|X\|,$ where $C$ is a positive constant depending only on the FNN $f$. In Figure \ref{picture1} we present the worst-case error. The worst-case error decreases as $N$ increases or $\delta$ decreases, but it may be large if the worst-case scenario corresponds to large $\|X\|$. To remove the dependence on $X$ and to focus on the dependence of $N$ and $\delta,$ we compute the average and represent the error bound as $\mathrm{E}_{X}\|f(X)-f_Q(X)\| \le C\times\mathrm{E}\|X\|\times\delta N^{-1}.$ Note that $\mathrm{E}\|X\|$ is now a constant depending only on the MNIST dataset, so by taking logarithm on both sides, we can write $\log \mathrm{E}_{X}[\|f(X)-f_Q(X)\|\times N/\delta] \le K$, for some constant $K>0$. In Figure \ref{picture2}, we present the values of $\log \mathrm{E}_{X}[\|f(X)-f_Q(X)\|\times N/\delta].$ We can see that the values are bounded and behave approximately constant as $N$ grows. These numerical results imply that our theoretical error bound is tight in terms of $\delta$ and $N.$

\begin{figure}[ht]
    \centering
    \begin{minipage}{0.49\textwidth}
        \centering
        \includegraphics[width=1\textwidth]{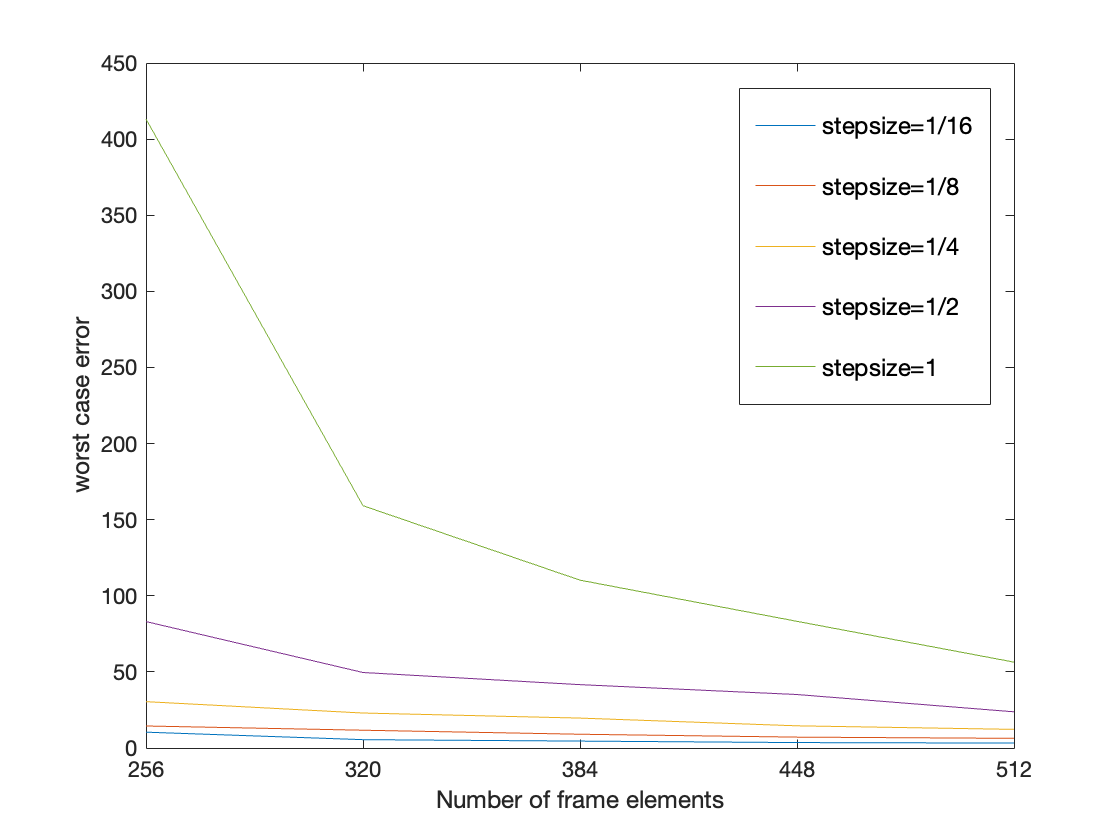} 
        \caption{Worst-case error $\|f(X)-f_Q(X)\|$ for FNN with 3 layers.}
        \label{picture1}
    \end{minipage}\hfill
    \begin{minipage}{0.49\textwidth}
        \centering
        \includegraphics[width=1\textwidth]{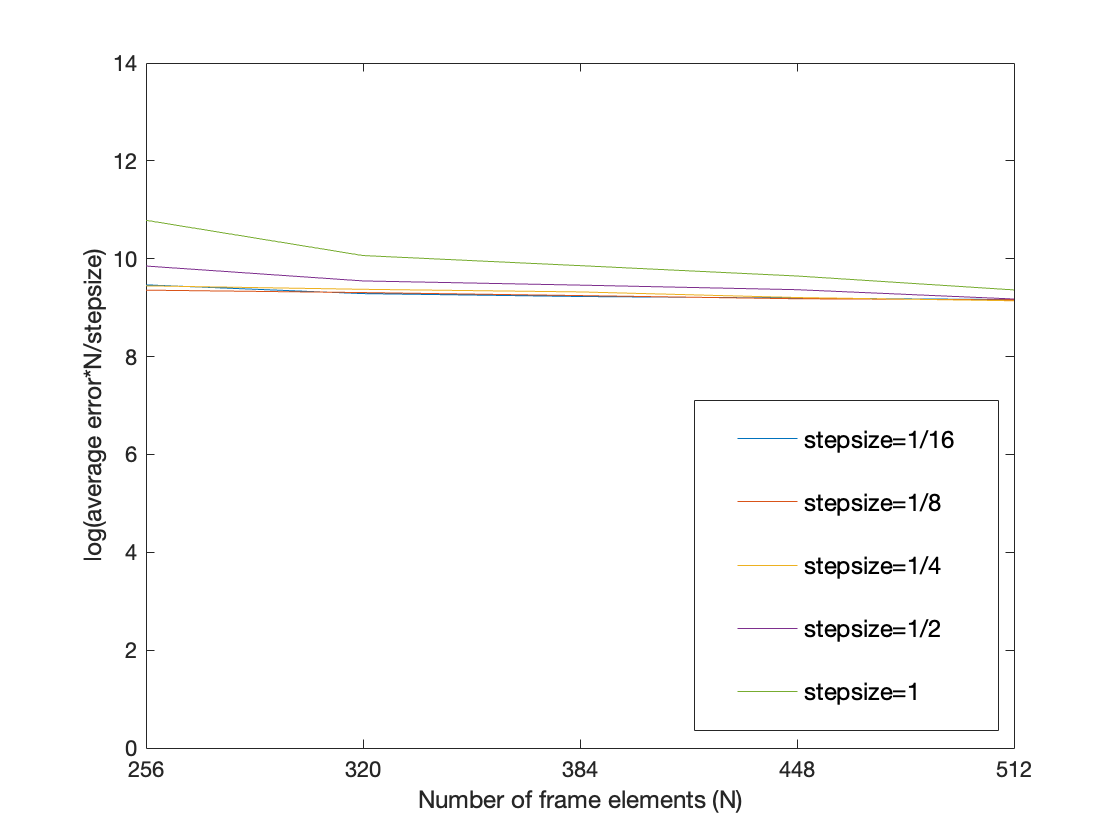} 
        \caption{$\log \mathrm{E}_{X}[\|f(X)-f_Q(X)\|\times N/\delta]$ for FNN with 3 layers.} 
        \label{picture2}
    \end{minipage}
\end{figure}

\subsection{Network with 2 Residual Blocks}
 Neural network in our experiment has a form of $g(x)=h^{[2]}\circ \sigma \circ z^{[2]}\circ \sigma \circ z^{[1]}\circ \sigma \circ h^{[1]}(x)$ where $h^{[1]}:\mathbb{R}^{784}\rightarrow \mathbb{R}^{256}$ and $h^{[2]}:\mathbb{R}^{256}\rightarrow\mathbb{R}^{10}$ are affine maps in (\ref{fnn}), and $z^{[1]},z^{[2]}:\mathbb{R}^{256}\rightarrow\mathbb{R}^{256}$ are residual blocks in (\ref{residual}). Our results are listed in Table \ref{table:2}. Quantizing with smaller $\delta$'s and larger $N$'s results in higher accuracy.

\begin{table}[ht!]
\centering
\begin{tabular}{|c| c| c| c|c|c|} 
 \hline
 $N$ & $\delta=1/16$ & $\delta=1/8$ & $\delta=1/4$ & $\delta=1/2$ & $\delta=1$  \\ 
 \hline
 256 & $97.44\pm0.28\%$ & $96.17\pm0.67\%$& $77.20\pm6.23\%$ & $15.61\pm2.89\%$ & $9.68\pm2.55\%$ \\ 
 \hline
 320 & $97.52\pm0.32\%$ & $97.09\pm0.35\%$& $92.54\pm1.19\%$ & $41.03\pm6.94\%$ & $11.67\pm2.58\%$ \\
 \hline
 384 & $97.60\pm0.19\%$ & $97.23\pm0.27\%$& $95.12\pm0.62\%$ & $66.27\pm8.36\%$ & $17.84\pm2.83\%$ \\
 \hline
 448 & $97.64\pm0.22\%$ & $97.37\pm0.28\%$& $96.35\pm0.42\%$ & $85.28\pm3.53\%$ & $28.96\pm6.04\%$ \\
 \hline
 512 & $97.66\pm0.20\%$ & $97.54\pm0.25\%$& $97.00\pm0.29\%$ & $92.30\pm3.11\%$ & $45.80\pm10.09\%$ \\
 \hline
\end{tabular}
\caption{Frame Quantization for Neural Network with 2 Residual Blocks. The accuracy for the trained neural network is $97.72\pm0.20\%$.}
\label{table:2}
\end{table}

\subsection{1-Bit Quantization}
1-bit quantization of neural network, in other words, \textit{binarization of neural network}, refers to a method of compressing weights or even activations into 1-bit numbers. For example, \cite{courbariaux2015binaryconnect} quantized the weights of the neural networks to $\{-1,1\}.$ Mathematical theories were also developed in 1-bit $\Sigma\Delta$ quantization for bandlimited signals \cite{gunturk2003one} and approximation capabilities of 1-bit neural networks \cite{gunturk2021approximation}.

In our method, since we are storing $q^i_{jk}$'s appearing in the quantized expansions, we will say 1-bit quantization is obtained when all $q^i_{jk}$'s are in $\{-a,a\}$ for some positive number $a>0.$ Note that once we choose appropriate $K\in \mathbb{N}$ and step size $\delta>0,$ our method forces $q^i_{jk}$'s as elements of the midrise quantization alphabet $A^{\delta}_K$ (\ref{alphabet}). If we recall the structure of $A^{\delta}_K=\{(-K+\frac{1}{2})\delta,(-K+\frac{3}{2})\delta,\cdots,-\frac{1}{2}\delta,\frac{1}{2}\delta,\cdots,(K-\frac{1}{2})\delta\},$ one can easily see that $A^{\delta}_K$ achieves a form of $\{-a,a\}$ if and only if $K=1.$ 

Therefore in our experiments for 1-bit quantization, we set $K=1$ for all layers and found the uniform step size $\delta$ that satisfies (\ref{howtochoosedelta}) for $i=1,\cdots,n.$ For both neural networks in the previous experiments, setting $\delta=8$ enables us to set $K=1$. Our 1-bit quantization results are in Table \ref{table:3} and Table \ref{table:4}.

Now we shall explain how our algorithm can benefit in storage while maintaining the accuracy. Consider 1-bit quantization for a neural network with 3 layers. We use a harmonic frame $H^{256}_{N}$ to quantize the first and second weight matrices $W_1$ and $W_2,$ the total bits that we need to store the quantized matrices $Q_1$ and $Q_2$ is $1\times784\times N + 1\times 256 \times N=1040N$ bits. The total bits that we need to represent the weight matrices $W_1$ and $W_2$ are $32\times256\times784+32\times256\times256=8192\times1040$ bits. Suppose that we use $N=7000$. Then the total bits that we can save is $8192\times1040-1040N=1192\times1040$ bits, but we have an average accuracy of $97.29\%$ for the quantized networks, where the average accuracy for the trained neural networks is $97.72\%.$ So we can even benefit in storage for a 1-bit quantization case with a small sacrifice in the accuracy. 
\begin{table}[ht!]
    \centering
    \begin{tabular}{|c|c|c|c|c|c|c|c|}
        \hline
          $N$ &1000  & 2000 & 3000 & 4000 & 5000 & 6000 & 7000\\
         \hline
         Average & 36.18\% & 74.09\% & 88.47\% & 94.62\% & 96.04\% & 96.83\% & 97.29\% \\
         \hline
         Standard Deviation &7.65 & 4.71 & 1.82 & 0.44 & 0.82 & 0.36 & 0.25\\
         \hline
    \end{tabular}
    \caption{1-Bit Frame Quantization for FNN with 3 layers}
    \label{table:3}
\end{table}
\vspace{-3em}
\begin{table}[ht!]
    \centering
    \begin{tabular}{|c|c|c|c|c|c|c|c|}
        \hline
          $N$ &1000  & 2000 & 3000 & 4000 & 5000 & 6000 & 7000\\
         \hline
         Average & 
         14.12\% & 31.52\% & 66.93\% & 86.47\% & 92.88\% & 95.28\% & 96.30\% \\
         \hline
         Standard Deviation & 2.52 & 6.56 & 5.69 & 4.04 & 1.44 & 0.64 & 0.50 \\
         \hline
    \end{tabular}
    \caption{1-Bit Frame Quantization for Neural Network with 2 Residual Blocks}
    \label{table:4}
\end{table}

\section{Conclusions and Future Work}

In this paper, we proposed a PTQ method with a rigorous mathematical analysis of the error estimates between a neural network and its quantized network. We use frame theory in neural network quantization for the first time. Also, we demonstrate that 1-bit quantization with our method has a benefit in storage while attaining accuracy near the accuracy of the original neural network. But still, there are many open questions left. Here, we list some of them.
\begin{enumerate}
    \item In this paper we only focus on using FUNTFs for $\mathbb{R}^d.$ Designing a quantization method using different types of frames would be a new task for the future.
    \item Our method uses first-order $\Sigma\Delta$ quantization, but one may consider using higher-order $\Sigma\Delta$ quantization. In this case, how much can we benefit from the error bound in terms of the number of frame elements? Also, which quantization rule should be used for higher-order schemes?
    \item To the best of our knowledge, this is the first attempt to adopt frame theory to neural network quantization. We were able to explain the accuracy and storage benefits, but we must admit we do not immediately see a connection between using frames and saving inference time. To develop a new frame quantization algorithm that reduces inference time, we may need to use some specific frame.
\end{enumerate}

    


\begin{thebibliography}{99}



\bibitem{banner2019post}
Banner, Ron, Yury Nahshan, and Daniel Soudry. \textit{Post training 4-bit quantization of convolutional networks for rapid-deployment.} Advances in Neural Information Processing Systems 32 (2019).

\bibitem{benedetto2010pointwise}
Benedetto, John J., and Onur Oktay. \textit{Pointwise comparison of PCM and $\Sigma\Delta$ quantization.} Constructive Approximation 32 (2010): 131-158.

\bibitem{benedetto2006second}
Benedetto, John J., Alexander M. Powell, and Özgür Yılmaz. \textit{Second-order Sigma–Delta} ($\Sigma\Delta$) \textit{quantization of finite frame expansions.} Applied and Computational Harmonic Analysis 20.1 (2006): 126-148.

\bibitem{benedetto2006sigma}
Benedetto, John J., Alexander M. Powell, and Özgür Yılmaz. \textit{Sigma-delta} ($\Sigma\Delta$) \textit{quantization and finite frames.} IEEE Transactions on Information Theory 52.5 (2006): 1990-2005.

\bibitem{benedetto2008complex}
Benedetto, John J., Onur Oktay, and Aram Tangboondouangjit. \textit{Complex Sigma–Delta quantization algorithms for
finite frames.} Contemp. Math. 464 (2008): 27–51.

\bibitem{casazza2012finite}
Casazza, Peter G., and Gitta Kutyniok, eds. \textit{Finite frames: Theory and applications.} Springer Science \& Business Media (2012).


\bibitem{courbariaux2015binaryconnect}
Courbariaux, Matthieu, Yoshua Bengio, and Jean-Pierre David. \textit{Binaryconnect: Training deep neural networks with binary weights during propagations.} Advances in Neural Information Processing Systems 28 (2015).

\bibitem{daubechies2003approximating}
Daubechies, Ingrid, and Ron DeVore. \textit{Approximating a bandlimited function using very coarsely quantized data: A family of stable sigma-delta modulators of arbitrary order.} Annals of Mathematics 158.2 (2003): 679-710.

\bibitem{gholami2022survey}
Gholami, Amir, et al. \textit{A survey of quantization methods for efficient neural network inference.} Low-Power Computer Vision. Chapman and Hall/CRC (2022): 291-326.

\bibitem{grohs2022mathematical}
Grohs, Philipp, and Gitta Kutyniok, eds. \textit{Mathematical aspects of deep learning.} Cambridge University Press (2022).

\bibitem{gunturk2004approximating}
Güntürk, C. Sinan. \textit{Approximating a bandlimited function using very coarsely quantized data: improved error estimates in sigma-delta modulation.} Journal of the American Mathematical Society 17.1 (2004): 229-242.

\bibitem{gunturk2003one}
Güntürk, C. Sinan. \textit{One‐bit sigma‐delta quantization with exponential accuracy.} Communications on Pure and Applied Mathematics: A Journal Issued by the Courant Institute of Mathematical Sciences 56.11 (2003): 1608-1630.

\bibitem{gunturk2021approximation}
Güntürk, C. Sinan, and Weilin Li. \textit{Approximation of functions with one-bit neural networks.} arXiv preprint arXiv:2112.09181 (2021).

\bibitem{gunturk2010sigma}
Güntürk, C. Sinan, et al. \textit{Sigma delta quantization for compressed sensing.} 2010 44th Annual Conference on Information Sciences and Systems (CISS). IEEE, 2010.

\bibitem{gunturk2013sobolev}
Güntürk, C. Sinan, et al. \textit{Sobolev duals for random frames and} $\Sigma\Delta$ \textit{quantization of compressed sensing measurements.} Foundations of Computational mathematics 13 (2013): 1-36.

\bibitem{guo2022squant}
Guo, Cong, et al. \textit{SQuant: On-the-fly data-free quantization via diagonal hessian approximation.} arXiv preprint arXiv:2202.07471 (2022).


\bibitem{he2016deep}
He, Kaiming, et al. \textit{Deep residual learning for image recognition.} Proceedings of the IEEE Conference on Computer Vision and Pattern Recognition (2016).

\bibitem{hubara2016binarized}
Hubara, Itay, et al. \textit{Binarized neural networks.} Advances in Neural Information Processing Systems 29 (2016).

\bibitem{krahmer2014sigma}
Krahmer, Felix, Rayan Saab, and \"Ozg\"ur Yilmaz. \textit{Sigma–delta quantization of sub-gaussian frame expansions and its application to compressed sensing.} Information and Inference: A Journal of the IMA 3.1 (2014): 40-58.

\bibitem{kummer2023adaptive}
Kummer, Lorenz, et al. \textit{Adaptive Precision Training (AdaPT): A dynamic quantized training approach for DNNs.} Proceedings of the 2023 SIAM International Conference on Data Mining (SDM). Society for Industrial and Applied Mathematics, 2023.


\bibitem{long2023recurrence}
Long, Ziang, Penghang Yin, and Jack Xin. \textit{Recurrence of optimum for training weight and activation quantized networks.} Applied and Computational Harmonic Analysis 62 (2023): 41-65.



\bibitem{lybrand2021greedy}
Lybrand, Eric, and Rayan Saab. \textit{A greedy algorithm for quantizing neural networks.} The Journal of Machine Learning Research 22.1 (2021): 7007-7044.

\bibitem{maly2023simple}
Maly, Johannes, and Rayan Saab. \textit{A simple approach for quantizing neural networks.} Applied and Computational Harmonic Analysis 66 (2023): 138-150.

\bibitem{nagel2020up}
Nagel, Markus, et al. \textit{Up or down? adaptive rounding for post-training quantization.} International Conference on Machine Learning. PMLR (2020).


\bibitem{nahshan2021loss}
Nahshan, Yury, et al. \textit{Loss aware post-training quantization.} Machine Learning 110.11-12 (2021): 3245-3262.

\bibitem{rastegari2016xnor}
Rastegari, Mohammad, et al. \textit{Xnor-net: Imagenet classification using binary convolutional neural networks.} European conference on computer vision. Cham: Springer International Publishing (2016).

\bibitem{saab2018quantization}
Saab, Rayan, Rongrong Wang, and \"Ozg\"ur Yılmaz. \textit{Quantization of compressive samples with stable and robust recovery.} Applied and Computational Harmonic Analysis 44.1 (2018): 123-143.

\bibitem{wang2008sigma}
Wang, Yang. \textit{Sigma–Delta quantization errors and the traveling salesman problem.} Advances in Computational Mathematics 28 (2008): 101-118.

\bibitem{wei2022qdrop}
Wei, Xiuying, et al. \textit{Qdrop: Randomly dropping quantization for extremely low-bit post-training quantization.} arXiv preprint arXiv:2203.05740 (2022).

\bibitem{yao2022zeroquant}
Yao, Zhewei, et al. \textit{Zeroquant: Efficient and affordable post-training quantization for large-scale transformers.} Advances in Neural Information Processing Systems 35 (2022): 27168-27183.

\bibitem{yin2018binaryrelax}
Yin, Penghang, et al. \textit{Binaryrelax: A relaxation approach for training deep neural networks with quantized weights.} SIAM Journal on Imaging Sciences 11.4 (2018): 2205-2223.

\bibitem{zhang2023spfq}
Zhang, Jinjie, and Rayan Saab. \textit{SPFQ: A Stochastic Algorithm and Its Error Analysis for Neural Network Quantization.} arXiv preprint arXiv:2309.10975 (2023).


\bibitem{zhang2023post}
Zhang, Jinjie, Yixuan Zhou, and Rayan Saab. \textit{Post-training quantization for neural networks with provable guarantees.} SIAM Journal on Mathematics of Data Science 5.2 (2023): 373-399.

\bibitem{zhang2024sigma}
Zhang, Jinjie, et al. \textit{Sigma-Delta and distributed noise-shaping quantization methods for random Fourier features.} Information and Inference: A Journal of the IMA 13.1 (2024): iaad052.

\end{thebibliography}

\end{document}